\documentclass{article}
\usepackage{ijcai21}

\usepackage{times}

\usepackage{soul}
\usepackage{url}
\usepackage[hidelinks]{hyperref}
\usepackage[utf8]{inputenc}
\usepackage[small]{caption}
\usepackage{graphicx}
\usepackage{amsmath}
\usepackage{amsthm}
\usepackage{booktabs}
\usepackage{amsfonts,amssymb}
\usepackage{algorithm}
\usepackage[noend]{algorithmic}
\usepackage{amsfonts,amssymb}
\usepackage{bm}
\usepackage{color}
\usepackage{appendix}
\usepackage{subfigure}

\graphicspath{{./figs/}}

\usepackage{balance}

\newtheorem{theorem}{Theorem}

\title{CFR-MIX: Solving Imperfect Information Extensive-Form Games \\ with Combinatorial Action Space}

\author{
Shuxin Li$^1$ 
\and
Youzhi Zhang$^{2}$\thanks{Equal contribution.}\and 
Xinrun Wang$^{1}$\footnotemark[1] \and 
Wanqi Xue$^1$ \And
Bo An$^1$\\
\affiliations
$^1$School of Computer Science and Engineering, Nanyang Technological University, Singapore\\
$^2$Department of Computer Science, Dartmouth College, USA\\
\emails
\{shuxin.li, xinrun.wang, boan\}@ntu.edu.sg,
youzhi.zhang@dartmouth.edu,
wanqi001@e.ntu.edu.sg
}

\begin{document}

\maketitle

\begin{abstract}
In many real-world scenarios, a team of agents coordinate with each other to compete against an opponent. The challenge of solving this type of game is that the team's joint action space grows exponentially with the number of agents, which results in the inefficiency of the existing algorithms, e.g., Counterfactual Regret Minimization (CFR). To address this problem, we propose a new framework of CFR: \textbf{CFR-MIX}. Firstly, we propose a new strategy representation that represents a joint action strategy using individual strategies of all agents and a consistency relationship to maintain the cooperation between agents. To compute the equilibrium with individual strategies under the CFR framework, we transform the consistency relationship between strategies to the consistency relationship between the cumulative regret values. Furthermore, we propose a novel decomposition method over cumulative regret values to guarantee the consistency relationship between the cumulative regret values. Finally, we introduce our new algorithm CFR-MIX which employs a mixing layer to estimate cumulative regret values of joint actions as a non-linear combination of cumulative regret values of individual actions. Experimental results show that CFR-MIX outperforms existing algorithms on various games significantly. 
\end{abstract}

\section{Introduction}
Extensive-form games provide a versatile framework capable of representing multiple agents, imperfect information, and stochastic events. Now there are many researches about solving two-player zero-sum extensive-games, such as computing Nash equilibria by linear programs ~\cite{shoham2008multiagent}, double oracle algorithms ~\cite{mcmahan2003planning,jain2011double}, and CFR ~\cite{zinkevich2008regret}. Meanwhile, these scalable algorithms have achieved many accomplishments. Double oracle algorithms ~\cite{mcmahan2003planning,jain2011double} have subsequently been applied to real-world attacker-defender scenarios. Heads-up limit hold’em poker was essentially solved ~\cite{bowling2015heads}. Later, significant progress has been made for heads-up no-limit hold’em poker based on CFR, such as Libratus ~\cite{brown2017libratus} and DeepStack ~\cite{moravvcik2017deepstack}. 

CFR is one of the most popular algorithms to solve imperfect-information extensive-form games which is an iterative algorithm to approximate a Nash equilibrium with repeated self-play between two regret-minimizing algorithms. Some sampling-based CFR variants ~\cite{lanctot2009monte,gibson2012generalized,lanctot2013monte} are proposed to solve large games effectively. Nowadays, neural network function approximation is applied to CFR to solve larger games. Deep CFR~\cite{brown2019deep}, Single Deep CFR~\cite{steinberger2019single} and Double Neural CFR~\cite{li2019double} are algorithms aiming to approximate CFR using deep neural networks. 

Here, we focus on a special type of two-player imperfect-information extensive-form zero-sum games, \textit{Team Adversary Games}, in which a team of cooperative agents plays against an adversary and all agents in the team share one utility function. This game model captures many real-world scenarios, such as many policemen coordinate with each other to catch an attacker ~\cite{basilico2017coordinating}. The challenge of solving this type of games is that the size of the team's joint action space is exponential with the number of agents. For example, if the number of agents is $8$ and every agent has $10$ actions, then the team's action space will be $10^8$. Therefore, it is impractical to solve this type of game with existing algorithms. CFR and its sampling-based variants use the tabular-form to record the joint action strategy which results in impracticality due to the limited memory. For algorithms using deep neural networks, e.g., Deep CFR and Double Neural CFR, it is ineffective to train the joint action strategy over the large action space. One way to avoid computing the joint action strategy is to let each agent compute its policy independently. However, this approach cannot promote the cooperative interaction between agents in the team and there is no theoretical guarantee to converge to the equilibrium in multi-player games \cite{abou2010using}.

To solve the \emph{exponential combinatorial action space problems} in Team-Adversary Games, we propose a new framework of CFR: \textbf{CFR-MIX}. Firstly, inspired by the idea of centralized training with decentralized execution \cite{rashid2018qmix}, we propose a novel strategy representation that represents the team's joint action strategy using the individual strategy of each agent. We also define a consistency relationship between these two strategy representations to maintain the cooperation among agents in the team, which guarantees that the equilibrium with new strategy representation is a special team-maxmin equilibrium. To compute Nash equilibrium with individual strategy representation under the CFR framework, we transform the consistency relationship between strategies to the consistency relationship between the cumulative regret values of joint actions and the cumulative regret values of individual actions. Furthermore, a novel decomposition method over cumulative regret values is proposed to guarantee the consistency relationship between the cumulative regret values. To implement the decomposition method, CFR-MIX employs a mixing layer that estimates cumulative regret values of joint actions as a non-linear combination of cumulative regret values of individual actions. To further improve the performance, the parameter sharing technique is applied among all agents in the team, which reduces the network parameters dramatically. Finally, experimental results show that CFR-MIX outperforms state-of-the-art algorithms significantly on games in different domains.
\section{Background}
\noindent \textbf{Extensive-Form Game} An Extensive-Form Game (EFG) ~\cite{shoham2008multiagent} can be formulated by a tuple ($N, H, A, P, \mathcal{I}, u$). $N = \{1,...,n\}$ is a set of players. $H$ is a set of histories (i.e., the possible action sequences). The empty sequence $\emptyset$ which is the root node of game tree is in $H$, and every prefix of a sequence in $H$ is also in $H$. $Z \subset H$ is the set of the terminal histories. $A(h) = \{a:(h,a) \in H\}$ is the set of available actions at a non-terminal history $h \in H$. $P$ is the player function. $P(h)$ is the player who takes an action at the history $h$, i.e., $P(h) \mapsto P \cup \{c\}$. $c$ denotes the ``chance player'', which represents stochastic events outside of the players' control. If $P(h) = c$ then chance determines the action taken at history $h$. Information sets $\mathcal{I}_i$ form a partition over histories $h$ where player $i \in \mathcal{N}$ takes action. Therefore, every information set $I_i \in \mathcal{I}_i$ corresponds to one decision point of player $i$ which means that $P(h_1) = P(h_2)$ and $A(h_1) = A(h_2)$ for any $h_1, h_2 \in I_i$. For convenience, we use $A(I_i)$ to represent the set $A(h)$ and $P(I_i)$ to represent the player $P(h)$ for any $h \in I_i$. For each player $i \in N$, a utility function is a mapping $u_i: Z \rightarrow \mathbb{R}$.

A player's behavior strategy $\sigma_i$ is a function mapping every information set of player $i$ to a probability distribution over $A(I_{i})$. A strategy profile $\sigma$ consists of a strategy for each player $\sigma_1, \sigma_2, ...$, with $\sigma_{-i}$ referring to all the strategies in $\sigma$ except $\sigma_i$. Let $\pi^{\sigma}(h)$ be the reaching probability of history $h$ if players choose actions according to $\sigma$. Given a strategy profile $\sigma$, the overall value to player $i$ is the expected payoff of the resulting terminal node, $u_i(\sigma) = \sum_{h \in Z}\pi^\sigma(h)u_i(h)$.

In this paper, we consider a type of imperfect-information extensive-form games, \emph{Team-Adversary Game}, that consists of an adversary $\mathbb{V}$ and one team $\mathbb{T}$. The team is formed by multiple agents $R = \{1, ..., n\}$. Every agent $i$ has its own available action set $A_{\mathbb{T}_i}$. Therefore, the team's action set includes all joint actions that are the combinations of all agents' individual actions, i.e., $A_\mathbb{T} = \times_{i \in R}A_{\mathbb{T}_i}$. Here, we focus on zero-sum games, i.e., $u_\mathbb{T}(z) = -u_\mathbb{V}(z)$ for $\forall z \in Z$. 
Nash equilibrium ~\cite{nash1950equilibrium} is adopted as our solution concept which is a strategy profile such that no player can get more reward by switching to a different strategy unilaterally. Formally, in Team-Adversary Games, a Nash Equilibrium (NE) is a strategy profile $\sigma$ where
\begin{align}
u_{\mathbb{V}(\mathbb{T})}(\sigma) \geq \max\nolimits_{\sigma'_{\mathbb{V}(\mathbb{T})} \in \sum_{\mathbb{V}(\mathbb{T})}} u_{\mathbb{V}(\mathbb{T})}(\sigma'_{\mathbb{V}(\mathbb{T})}, \sigma_{\mathbb{T}(\mathbb{V})}) \nonumber
\end{align}
An approximation of a Nash Equilibrium or $\epsilon$-Nash equilibrium is a strategy profile $\sigma$ where
\begin{align}
    u_{\mathbb{V}(\mathbb{T})}(\sigma) + \epsilon \geq \max\nolimits_{\sigma^{'}_{\mathbb{V}(\mathbb{T})} \in \sum_{\mathbb{V}(\mathbb{T})}} u_{\mathbb{V}(\mathbb{T})}(\sigma^{'}_{\mathbb{V}(\mathbb{T})}, \sigma_{\mathbb{T}(\mathbb{V})}) \nonumber
\end{align}
\noindent \textbf{Counterfactual Regret Minimization} \textit{CFR} is a family of iterative algorithms that are the most popular approach to approximately solve large imperfect-information games ~\cite{zinkevich2008regret}. To define the concept, first consider repeatedly playing games. Let $\sigma^t_i$ be the strategy used by player $i$ on round $t$. We define $u_i(\sigma, h)$ as the expected utility of player $i$ given that the history $h$ is reached and then all players play according to strategy $\sigma$ from that point on. Let's define $u_i(\sigma, h\cdot a)$ as the expected utility of player $i$ given that the history $h$ is reached and then all players play according to strategy $\sigma$ except that player $i$ selects action $a$ in the history $h$. Formally, $u_i(\sigma,h)=\sum_{z \in Z}\pi^\sigma(h,z)u_i(z)$ and $u_i(\sigma, h\cdot a)=\sum_{z \in Z}\pi^\sigma(h \cdot a,z)u_i(z)$.

The \textit{counterfactual value} $u_i^\sigma(I)$ is the expected value of an information set $I$ given that player $i$ tries to reach it. This value is the weighted average of the value of each history in an information set. The weight is proportional to the contribution of all players other than $i$ to reach each history. Thus, $u_i^\sigma(I) = \sum_{h\in I} \pi^\sigma_{-i}(h)\sum_{z \in Z}\pi^\sigma(h,z)u_i(z).$

For any action $a \in A(I)$, the counterfactual value of an action $a$ is $u_i^\sigma(I,a)= \sum_{h\in I} \pi^\sigma_{-i}(h)\sum_{z \in Z}\pi^\sigma(h \cdot a,z)u_i(z)$.
The \textit{instantaneous regret} for action $a$ in information set $I$ on iteration $t$ is $r^t(I,a)=u_{P(I)}^{\sigma^t}(I,a)-u_{P(I)}^{\sigma^t}(I)$. The cumulative regret for action $a$ in $I$ on iteration $T$ is $R^T(I,a) =\sum_{t=1}^{T}r^t(I,a)$. 
In CFR, players use \textit{Regret Matching} to pick a distribution over actions in an information set in proportion to the positive cumulative regret on those actions. Formally, on iteration $T+1$, player $i$ selects actions $a\in A(I)$ according to probability
\begin{equation}
\sigma^{T+1}(I,a)=
\begin{cases}
\frac{R^T_{+}(I,a)}{\sum_{b \in A(I)}R^T_{+}(I,b)} & \text{if $\sum_{b \in A(I)}R^T_{+}(I,b) >0$} \\
\frac{1}{|A(I)|}& \text{otherwise} 
\end{cases}\nonumber
\end{equation}
where $R^T_{+}(I,a)=\max(R^T(I,a),0)$ because we often mostly concern about cumulative regret when it is positive.
If a player plays according to CFR in every iteration, then on iteration $T$, $R^T(I) \leq \Delta_i\sqrt{|A_i|}\sqrt{T}$ where $\Delta_i= \max_z u_i(z)-\min_zu_i(z)$ is the range of utility of player $i$. Moreover, 
$R^T_i \leq \sum_{I\in \mathcal{I}_i}R^T(I) \leq |\mathcal{I}_i|\Delta_i\sqrt{|A_i|}\sqrt{T}$. 
Therefore, as $T \rightarrow \infty$, $\frac{R^T_i}{T} \rightarrow 0$. In two-player zero-sum games, if both players' average regret  $\frac{R^T_i}{T} \leq \epsilon$, then their average strategies $(\overline{\sigma}^T_1,\overline{\sigma}^T_2)$ form a $2\epsilon$-equilibrium~\cite{waugh2009abstraction}. 

To solve large games effectively, some sampling-based CFR algorithms are proposed, such as external sampling, outcome sampling ~\cite{lanctot2009monte}, probe sampling ~\cite{gibson2012generalized} and other reduce variance sampling algorithms~\cite{schmid2019variance,steinberger2020dream}. External sampling algorithm has the lowest variance. However, it runs slowly because it traverses all actions of one player every iteration. Although outcome sampling algorithm runs faster than external sampling algorithm, it has large variance which influences the convergence rate. Compared to outcome sampling algorithm, probe sampling algorithm attempts to reduce variance by replacing ``zeroed-out'' counterfactual values of non-sampled actions with closer estimates of the true counterfactual values. In this paper, we adopt the probe sampling algorithm to traverse the game tree to collect regret value data.
 
\section{Strategy Representation}
In this paper, we focus on solving the Team-Adversary Games in which a team of agents plays against an adversary and all agents share one utility function. Note that the team's joint action space grows exponentially with the number of team players. Thus, it is impractical to solve this type of games with existing approaches due to the large action space. To address this problem, we propose a novel strategy representation, which significantly reduces the strategy space. 
\subsection{Individual Strategy Representation}
In this section, we introduce a new strategy representation to denote the team's joint action strategy using the individual action strategies of all the team players. Different from the exponential combinatorial action space, the new individual action space of all the agents is linear with the number of agents. It is worth noting that all agents share the same history of the team $\mathbb{T}$, i.e., the player function can be extended as $P(h) = \mathbb{T} = \{1, 2, 3,..., n\}$. In other words, all the agents share the same information set of the team $\mathbb{T}$. Therefore, the new strategy representation of the team is defined as $f_\mathbb{T} = (\sigma_1, \sigma_2, ..., \sigma_n)$, where $\sigma_i$ is agent $i$'s strategy which maps each information set of player $\mathbb{T}$ to a probability distribution over agent $i$'s action space $A_{\mathbb{T}_i}$. For clarity, we call this new representation as \textit{individual strategy representation} and denote the old team strategy representation by \textit{joint strategy representation}. Given a strategy profile $\sigma = (\sigma_{\mathbb{V}}, f_{\mathbb{T}}) = (\sigma_{\mathbb{V}}, \sigma_1,...,\sigma_n)$, the expected payoff of every agents is the same one $u_a(\sigma)=\sum_{h\in Z}u_\mathbb{T}(h)\pi^{\sigma}(h)$. $\pi^{\sigma}(h)$ is the reaching probability of history $h$ if team players follow the strategy profile $\sigma$ which means each agent $i$ chooses its action according to its strategy $\sigma_i$ independently.
\subsection{Strategy Consistency}
Despite that the individual strategy representation reduces the search space efficiently, it cannot represent all join action strategy space. To clarify the difference between these two strategy representations, we first propose a consistency relationship between the two strategy representations which can be defined as
\begin{align}
     \boldsymbol{\sigma}_\mathbb{T}(I, \boldsymbol{a}) = \sigma_1(I, a_1)\sigma_2(I, a_2)...\sigma_n(I, a_{n}), \label{eq: strategy_consistency}
\end{align}
where $\boldsymbol{\sigma}_\mathbb{T}(\cdot)$ and $\sigma_i(\cdot)$ are the probability of a joint action and the probability of an individual action, respectively. In short, 
the joint strategy representation can be decomposed into the individual strategy representation.  
\begin{figure}
    \centering
    \includegraphics[width=0.9\columnwidth]{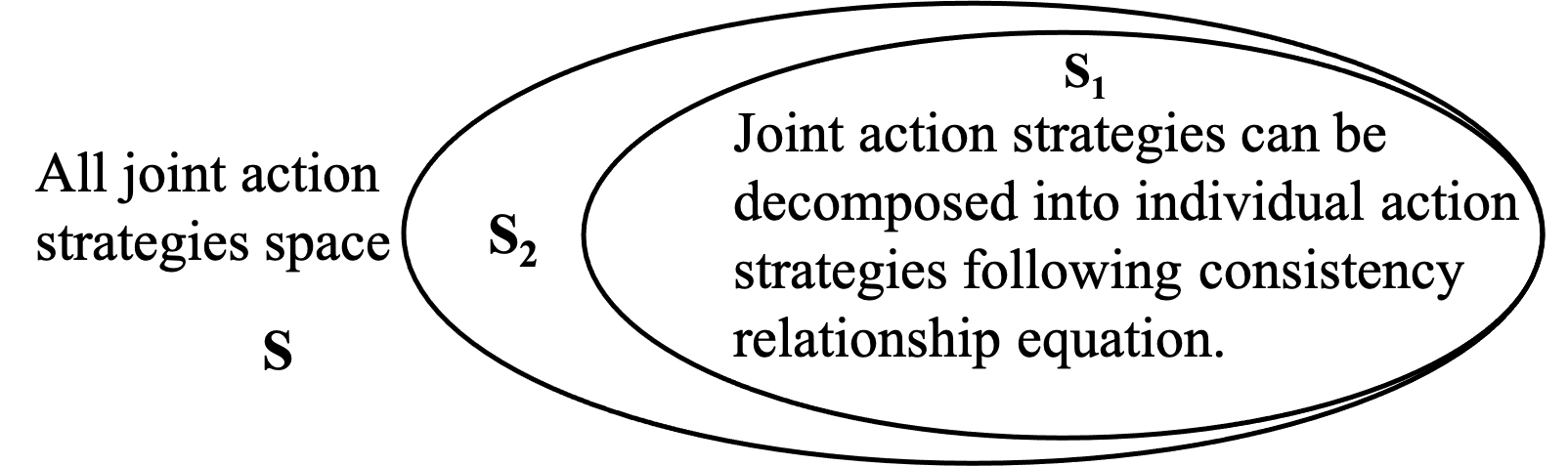}
    \caption{The relationship between joint action strategy space and individual action strategy space}
    \label{fig:my_label}
\end{figure}
Note that the consistency relationship divides the whole joint action space $S$ into two parts $S_1, S_2$ (Fig.~\ref{fig:my_label}). We set the strategy space of adversary as $S_V$ and $(S, S_V)$ forms a two-player zero-sum game. Under our new strategy representation, the strategy space of the team becomes $S_1$, and $(S_1, S_V)$ forms a two-player zero-sum game with the consistency relationship holds which is the game we focus on in this paper. The difference between these two games is that in game $(S_1, S_V)$, each team's joint strategy can be decomposed into the agents' individual strategies. Obviously, there are some joint strategies in $S$ that cannot be decomposed ($S_2$). For example, a joint action strategy $\boldsymbol{\sigma}=(0.5, 0, 0, 0.5)$ is based on the joint action set $\{(a_1, b_1),(a_1, b_2), (a_2, b_1), (a_2, b_2)\}$. We cannot get the corresponding individual strategy representation following Eq.~(\ref{eq: strategy_consistency}). Therefore, if the team's NE strategy of $(S, S_V)$ is in the strategy space $S_1$, the NE of $(S_1, S_V)$ equals the NE of $(S, S_V)$ (See Theorem~\ref{theorem}). Otherwise, the team's NE strategy of $(S_1, S_V)$ is a special team-maxmin equilibrium strategy that can be decomposed into individual strategy which is different from the NE of $(S, S_V)$. In this case, the team may lose some utility after reducing strategy space. Because the team's joint strategy always get not lower utility than the team's strategy which can be decomposed. However, experimental results show that solving the game $(S_1, S_V)$ can get good strategies faster than solving the large game $(S, S_V)$. The reason could be that the exponential action space significantly influences the computing process of the join action strategy significantly. 
\begin{theorem} 
\label{theorem}
Under the new strategy representation, the Nash Equilibrium strategy profile between the adversary and the team keeps unchanged if Eq.~(\ref{eq: strategy_consistency}) holds for the team's Nash equilibrium strategy. \footnote{The complete proof is in the Appendix.}
\end{theorem}
\section{CFR-MIX}
Now we introduce a novel algorithm CFR-MIX to compute an NE under Team-Adversary game with strategy consistency $(S_1,S_V)$. Firstly, we introduce the consistency relationship between cumulative regret values which is transformed from Eq.~(\ref{eq: strategy_consistency}). Secondly, a novel decomposition method over cumulative regret values is proposed based on the consistency relationship. Then, we introduce our algorithm CFR-MIX in which the decomposition method is implemented by a mixing layer. Finally, we provide theoretical analysis.
\subsection{Regret Consistency}
For extensive-form games, one of the most popular algorithms to compute NEs is CFR-based algorithm which leverages regret-matching to compute strategy according to cumulative regret values. In other words, the probability of each action is related to the cumulative regret value of the action. 
Therefore, to compute the individual strategy for each agent, we need to know the cumulative regret values of its actions explicitly. However, all the agents share the same team utility function and we can only get the regret values of joint actions. To this end, we propose a decomposition method over cumulative regret values. Before describing the decomposition method, we first transform the consistency relationship between strategies into the consistency relationship between cumulative regret values which can be defined as follows:
\begin{align}
    \forall \boldsymbol{a}, \quad \boldsymbol{\pi}_\mathbb{T}(R_{tot}(I, \boldsymbol{a})) & = \prod\nolimits_{i=1}^{n}\pi_i(R_{i}(I, a_i)),
    \label{eq: cumulative regret_consistency}
\end{align}
where $\pi_i(\cdot)$ and $\boldsymbol{\pi}_\mathbb{T}(\cdot)$ are the probabilities of individual action and joint action, i.e., $\pi_i(R(I,a)) = \sigma_i(I,a)$. $\boldsymbol{a} = (a_1, ..., a_{n})$ is the joint action composed by every agent's action. $R_{tot}$ and $R_i$ are the cumulative regret values of a joint action and an individual action, respectively. Since these probabilities are computed using regret-matching based on cumulative regret value, we can reformulate this consistency relationship as
\begin{align}
    \frac{R_{tot}(I, \boldsymbol{a})_+}{\sum\nolimits_{\forall \boldsymbol{a}^{'}}R_{tot}(I, \boldsymbol{a}^{'})_+} = \prod\nolimits_{i=1}^{n}\frac{R_i(I, a_i)_+}{\sum\nolimits_{b\in A_i(I)}R_i(I, b)_+} \label{eq.regret}
\end{align}
when $\sum\nolimits_{\forall \boldsymbol{a}^{'}}R_{tot}(I, \boldsymbol{a}^{'})_+ > 0$ and $\sum\nolimits_{b\in A_i(I)}R_i(I, b)_+ > 0$. If $\sum\nolimits_{\forall \boldsymbol{a}^{'}}R_{tot}(I, \boldsymbol{a}^{'})_+ \leq 0$, the cumulative regret values of individual actions need to satisfy:
$\sum\nolimits_{b\in A_i(I)}R_i(I, b)_+ \leq 0 $, $\forall i\in\{1,..n\}$.
The consistency relationship guarantees that the probability of a joint action equals the product of the probabilities of all individual actions selected by each player. 

\subsection{Product-Form Decomposition}
To guarantee these consistency relationships (Eqs.~(\ref{eq: strategy_consistency}-\ref{eq.regret})), we propose a product-form decomposition method over cumulative regret values which can be defined as follows:
\begin{align}
    \forall \boldsymbol{a}, \quad R_{tot}(I, \boldsymbol{a}) = \prod\nolimits_{i=1}^{n}R_i(I, a_i). \label{decomposition}
\end{align}
Here, we set $R_{tot} = 0$ if $R_{tot}\leq 0$ following the setting of regret-matching+ ~\cite{tammelin2015solving} which is a regret-minimizing algorithm that operates very similarly to regret-matching. Compare with the regret matching which ignores actions that have an cumulative negative regret, regret-matching+ actively resets any cumulative negative regret back to zero. In the remaining paper, we consider that all cumulative regret values are non-negative.
\begin{theorem}
If product-form decomposition (Eq.~(\ref{decomposition})) holds, the consistency relationship between the joint action strategy and the individual strategy (Eq.~(\ref{eq: strategy_consistency})) can be guaranteed.
\end{theorem}
\begin{proof}
When $\sum_{\forall \boldsymbol{a}^{'}}R_{tot}(I, \boldsymbol{a}^{'}) > 0$, we can get
\begin{align}
    \sigma(I, \boldsymbol{a}) & =  \frac{R_{tot}(I, \boldsymbol{a})}{\sum\nolimits_{\forall \boldsymbol{a}^{'}}R_{tot}(I, \boldsymbol{a}^{'})} = \frac{\prod\nolimits_{i=1}^{n}R_i(I, a_i)}{\sum\nolimits_{\forall \boldsymbol{a}^{'}}R_{tot}(I, \boldsymbol{a}^{'})},
\end{align}
\begin{align}
      \prod\nolimits_{i=1}^{n} \sigma(I, a_i) & = \prod\nolimits_{i=1}^{n}\frac{R_i(I,a_i)}{\sum\nolimits_{a_i^j\in A_i(I)}R_i(I, a_i^j)} \\
        & = \frac{\prod\nolimits_{i=1}^{n}R_i(I, a_i)}{\prod\nolimits_{i=1}^{n}\sum\nolimits_{a_i^j\in A_i(I)}R_i(I, a_i^j)}.  \nonumber
\end{align}
These equations are derived by the regret matching+ and product-form decomposition equation.
\begin{align}
    &\prod\nolimits_{i=1}^{n}\sum\nolimits_{a_i^j\in A_i(I)}R_i(I, a_i^j) \\
    = &\prod\nolimits_{i=1}^n[R_i(I, a_i^1)+R_i(I, a_i^2)+ ... +
    R_i(I,a_i^{|A_i(I)|})] \nonumber \\
    = &\prod\nolimits_{i=1}^n R_i(I, a_{i}^1) +\prod\nolimits_{i=1}^{n-1}R_i(I, a_{i}^1)R_{n-1}(I, a_{n-1}^2)+ \nonumber \\
    &... + \prod\nolimits_{i=1}^{n}R_i(I, a_{i}^{|A_i(I)|}) \nonumber \\
    = &R_{tot}(I, \boldsymbol{a}_1) + ... +R_{tot}(I, \boldsymbol{a}_{|A_1(I)|...|A_{n}(I)|})  \nonumber \\
    = &\sum\nolimits_{\forall \boldsymbol{a}^{'}}R_{tot}(I, \boldsymbol{a}^{'}). \nonumber
\end{align}
The above derivation is following the rules of polynomial multiplication and product-form decomposition equation. Therefore, the Eq.~(\ref{eq: strategy_consistency}) holds.
\end{proof}


\begin{figure}[t] 
\centering 

\includegraphics[width=0.46\textwidth]{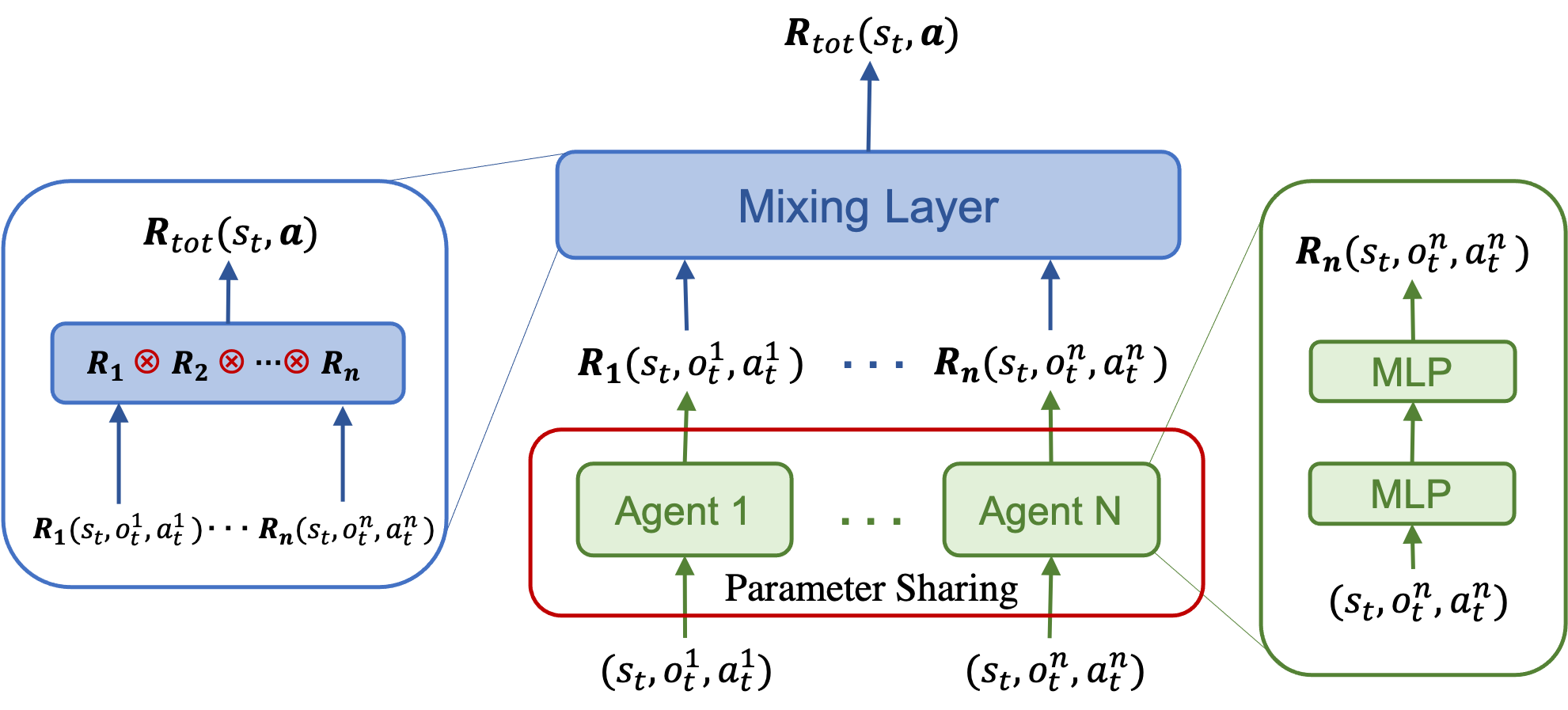} 
 
\caption{Architecture of team's cumulative regret neural network}
\label{fig: cfr_mix} 
\end{figure}

\subsection{Mixing Layer}
This section introduces our algorithm, CFR-MIX, which employs a mixing layer to implement the product-form decomposition under the Double Neural CFR framework. The reasons that we do not implement decomposition in tabular CFR framework are described as follows. If we use tabular CFR to compute every cumulative regret value of individual action according to Eq.~(\ref{decomposition}), we will get a nonlinear equation which is difficult to be solved efficiently. For example, in a Team-Adversary Game, there are two agents $R = \{1, 2\}$ in the team. At an information set $I_\mathbb{T}$, agent $1$ has two actions $\{a_{11}, a_{12}\}$ and agent $2$ has two actions $\{a_{21}, a_{22}\}$. We only know the cumulative regret values of all joint actions $R_{tot}(I, \boldsymbol{a})$, where $\boldsymbol{a} \in \{(a_{11}, a_{21}), (a_{11}, a_{22}), (a_{12}, a_{21}), (a_{12}, a_{22})\}$. Then we need to compute the cumulative regret value of each individual action following the equations:
$$ \left\{
\begin{aligned}
R_{tot}(I, (a_{11}, a_{21})) &= & R_1(I, a_{11}) \cdot R_2(I, a_{21})  \\
R_{tot}(I, (a_{11}, a_{22})) &= & R_1(I, a_{11}) \cdot R_2(I, a_{22})  \\
R_{tot}(I, (a_{12}, a_{21})) &= & R_1(I, a_{12}) \cdot R_2(I, a_{21}) \\
R_{tot}(I, (a_{12}, a_{22})) &= & R_1(I, a_{12}) \cdot R_2(I, a_{22}) 
\end{aligned}
\right.
$$
Solving the above equations is a resource intensive task, especially when the game has a large number of agents. Therefore, we adopt to combine CFR with deep neural networks and utilize the great capabilities of neural networks in function approximation properties (e.g., Double Neural CFR).

Our CFR-MIX algorithm is built on the Double Neural CFR framework, considering it can represent cumulative regret explicitly. Specifically, CFR-MIX uses the Probe sampling algorithm to traverse the game tree to collect training data (the data of cumulative regret values and strategies). In the sampling process, when it turns the team player to act, every agent plays according to its strategy and composes their strategies into the joint action strategy which is used to compute regret values only. Because all agents share one team utility function, it computes regret values for all joint actions based on the joint action strategy and joint action values. Then, CFR-MIX trains one cumulative regret value network and one average strategy network for each player.

The key to our method is the new cumulative regret network for the team which is shown in Fig.~\ref{fig: cfr_mix}. To decompose the global cumulative regret values of joint actions, we design a mixing layer as a product operation on the cumulative regret values of all individual actions. The Agent network takes the information sets, agent $i$'s observations and agent $i$'s available actions as input and outputs the cumulative regret value of the agent $i$'s actions. The mixing layer takes all the outputs of the Agent network as input and outputs the cumulative regret value of the corresponding joint action. In this way, the mixing layer guarantees that the consistency between cumulative regret values of individual action and joint action. Therefore, the mixing layer can also guarantee the strategy consistency relationship. When training this cumulative regret network, we use the cumulative regret values of joint actions which are the sum of: (1) the regret values obtained from the sampling process and (2) the estimated cumulative regret values obtained from the cumulative regret network in the last iteration. 

To get the individual strategy of each agent, we only use the Agent network to estimate the cumulative regret value for each agent's action and compute the agent's strategy using regret-matching+ based on the estimated cumulative regret value. Therefore, CFR-MIX only needs to traverse the action spaces of all the agents which are linear with the number of agents. On the contrary, to get the joint strategy, Double Neural CFR needs to traverse the whole joint action space which is exponential with the number of agents. Furthermore, in CFR-MIX, we adopt the parameter sharing technique which means that all agents share the same Agent network. The parameter sharing can reduce the parameters and improve the performance significantly. For the average strategy network, we also only use one network for all agents. The training data of the strategy network are all agents' strategies. However, in the Double Neural CFR algorithm, the training data of the average strategy network are the team strategies defined over all the joint actions. 

\subsection{Convergence Analysis}
Recall that we focus on a subset of joint action strategy space which can be decomposed into individual action strategies. Therefore, we compute a new equilibrium of game $(S_1, S_V)$ with proposed CFR-MIX and provide a bound of regret under mild conditions based on Deep CFR ~\cite{brown2019deep}. 
\begin{theorem}
Let $T$ denote the number of CFR-MIX iterations, $|A|$ the maximum number of actions at any infoset and $K$ the number of traversals per iteration. Let $L_{\mathcal{R}}^t$ be the average MSE loss for $\mathcal{R}_p(I, a|\theta^t)$ on a sample in $M_{r,p}$ at iteration $t$, and let $L_{\mathcal{R}^*}^t$ be the minimum loss achievable for any function $\mathcal{R}$. Let $L_{\mathcal{R}}^t-L_{\mathcal{R}^*}^t \leq \epsilon_{\boldsymbol{L}}$. If the value memories are sufficiently large and \textbf{ Eq.~(\ref{decomposition}) holds}, then with probability $1-\rho$ total regret of player $p$ at time $T$ is bounded by 
\begin{equation}
    R_p^T \leq (1+\frac{\sqrt{2}}{\sqrt{\rho K}})\Delta |I_p|\sqrt{|A|}\sqrt{T}+4T|I_p|\sqrt{|A|\Delta \epsilon_{\boldsymbol{L}}} \nonumber
\end{equation}
As $T \rightarrow \infty$, the average regret $\frac{R^T_p}{T}$ is bounded by $4|\mathcal{I}_p|\sqrt{|A|\Delta \epsilon_{\boldsymbol{L}}}$
with high probability.
\end{theorem}
\begin{proof}
Compared with the proof of Deep CFR, the only difference is that the agent in the team plays the individual strategy instead of the joint action strategy. The establishment of Eq.~(\ref{decomposition}) guarantees the consistency relationship of strategies (Eq.~(\ref{eq: strategy_consistency})) which means that under the new strategy representation, the Nash Equilibrium keeps unchanged (Theorem 1). Therefore, the regret bound for the new representation is same as the joint strategy representation.\footnote{The complete proof is in the Appendix.}
\end{proof}
\section{Evaluation}
To verify the effectiveness of CFR-MIX, we deploy Double Neural CFR and Deep CFR as baselines and compare these baselines with CFR-MIX and Deep CFR with mixing layer. We also compare CFR-MIX with the algorithm that every agent learns its strategy independently (Individual Deep CFR). Furthermore, to verify the effectiveness of the parameter sharing technique, we evaluate the performance of CFR-MIX with and without parameter sharing. Without loss of generality, we select two different domain games. Experiments are performed on a server with a 10 core 3.3GHz Intel i9-9820X CPU and a NVIDIA RTX 2080 Ti GPU.
\subsection{Goofspiel}
\textit{Goofspiel}~\cite{ross_1971} is a bidding card game where players have a hand of cards numbered 1 to $K$, and take turns secretly bidding on the top point-valued card in a point card stack using the cards in their hands (Fig.~\ref{fig: goof}(a)). In each bidding process, the player with the higher point of the card will get the point of the top point-valued card. Finally, the player with the highest total points wins. Here, we use an imperfect informational version which is widely used~\cite{lanctot2014search,lisy2015online,brown2019solving}: players can only observe the results of each bid while cannot know the cards used to bid. Meanwhile, we consider the special two-player case in which several players form a team to play against an opponent. If one of the team players wins, the whole team wins. Here, as shown in Fig.~\ref{fig: goof}, we use different number of cards (6C, 10C and 13C) and team players (4P and 6P).
\begin{figure}[t]
\centering
\subfigure[Goofspiel Game]{
\includegraphics[width=0.45\columnwidth, height=83pt]{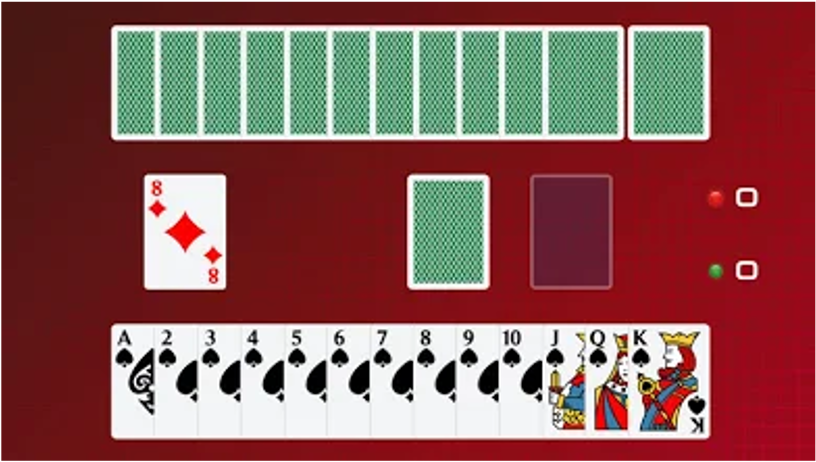}
 }
 \subfigure[6C4P6R]{
 \includegraphics[width=0.45\columnwidth]{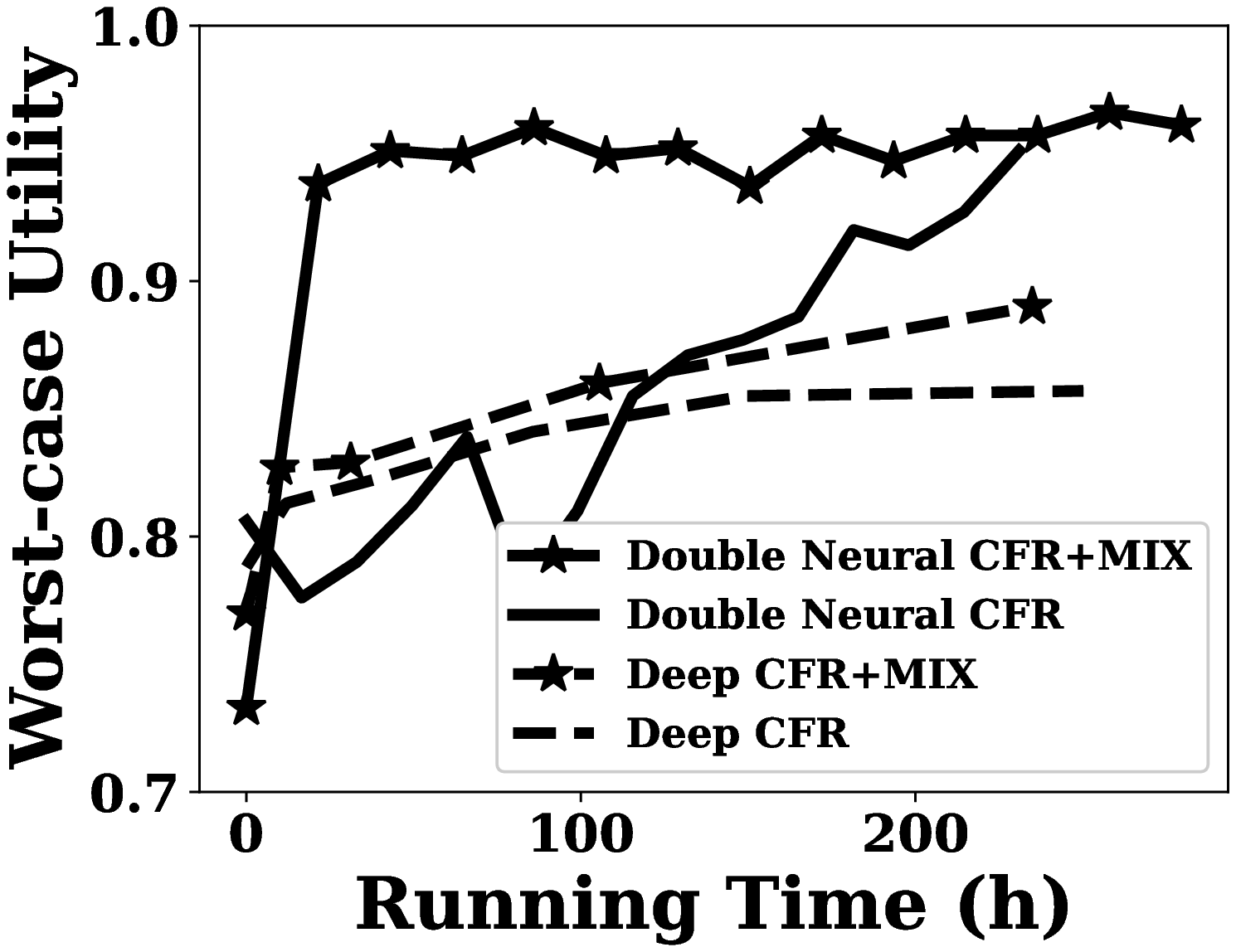}}
\subfigure[10C6P(2/3/4)R]{
\includegraphics[width=0.45\columnwidth]{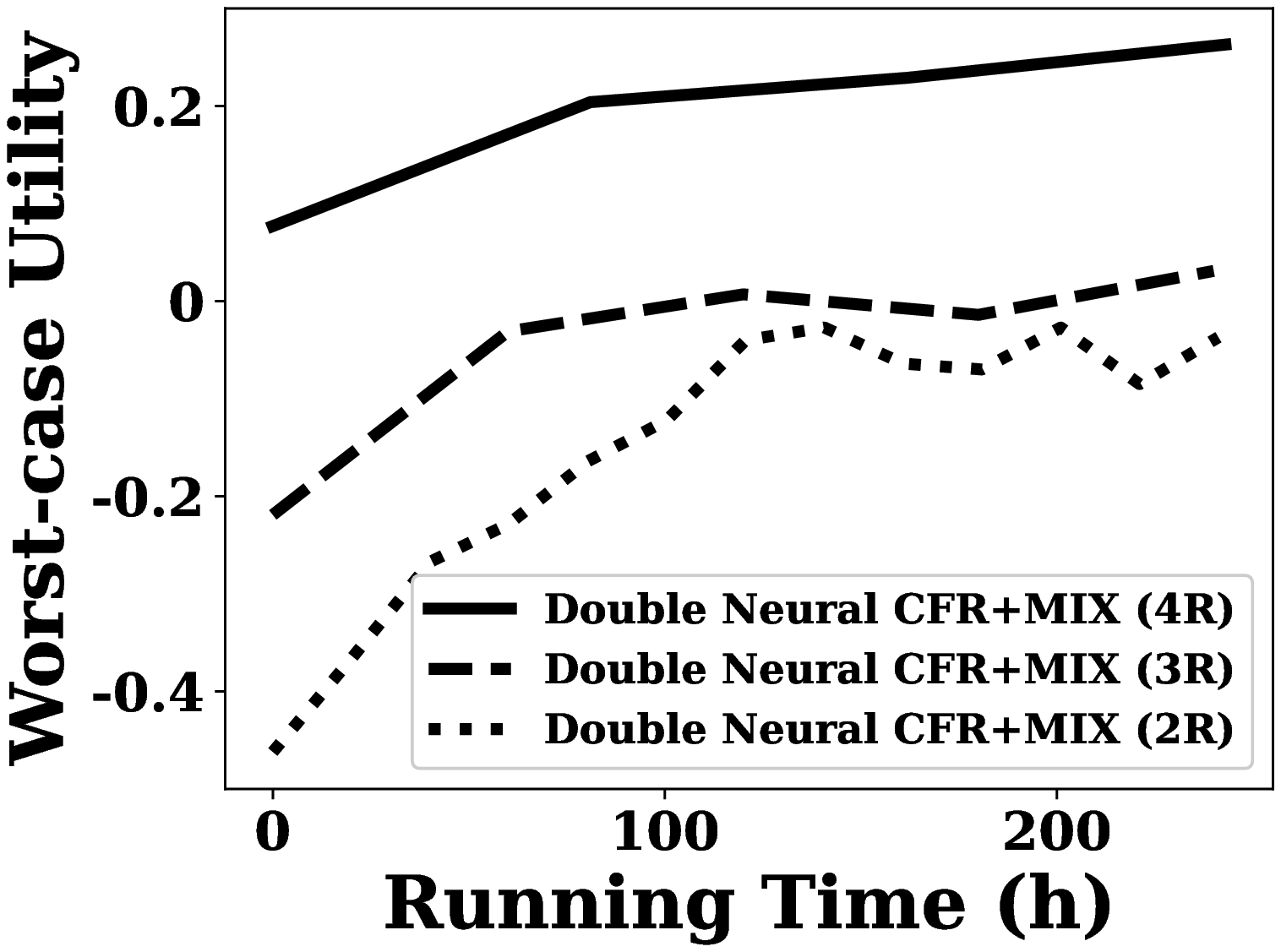}
 }
\subfigure[13C6P2R]{
 \includegraphics[width=0.45\columnwidth]{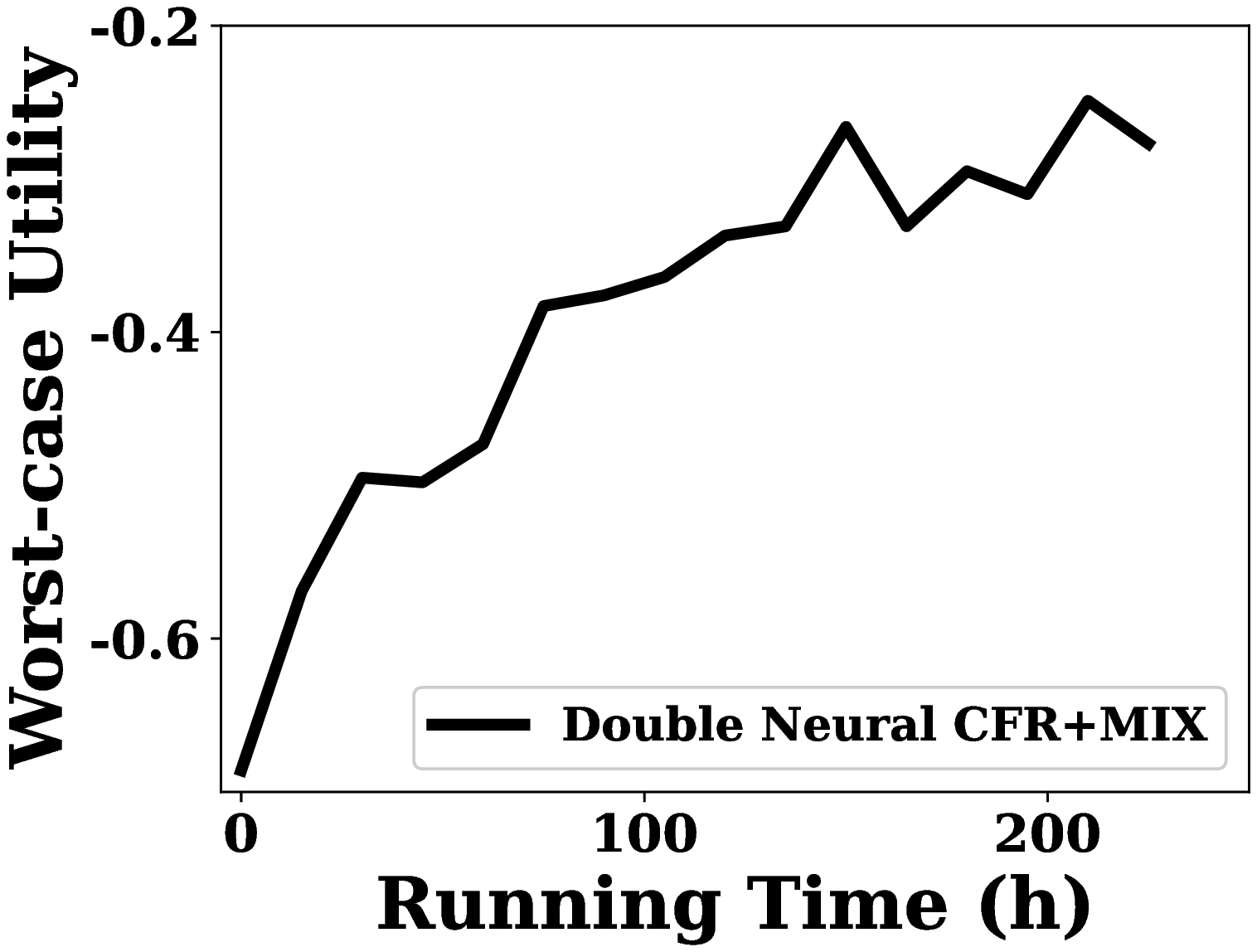}}
\caption{Goofspiel Games
(C: card, P: team player, R: round)}
\label{fig: goof} 
\end{figure}
Note that Double Neural CFR + MIX refers to CFR-MIX and Deep CFR + MIX denotes the Deep CFR algorithm with mixing layer. We can see that the algorithms with mixing layer outperform the others in 6C4P6R game. In 10C6P and 13C6P games, the joint action space is very large at some information sets (about $10^6$ and $13^6$), and only CFR-MIX can get satisfactory strategies in a limited time. 
\begin{figure}[t]
\centering
    \subfigure[3*3 grid, 1 vs 2]{
        \includegraphics[width=0.45\columnwidth]{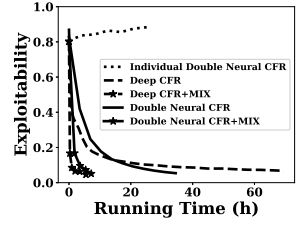}
    }
    \subfigure[5*5 grid, 1 vs 4]{
        \includegraphics[width=0.45\columnwidth]{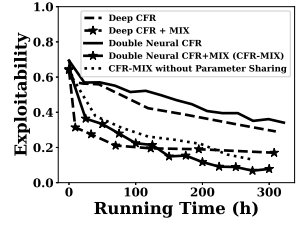}
    }
    \subfigure[15*15 grid, 1 vs 4]{
        \includegraphics[width=0.45\columnwidth]{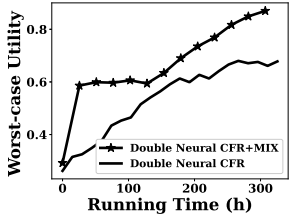}
    }
    \subfigure[15*15 grid, 1 vs 8]{
        \includegraphics[width=0.47\columnwidth]{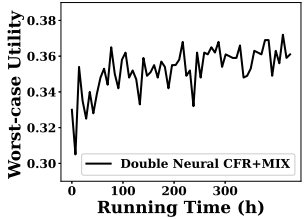}
    }

\caption{NEST Games}
\label{fig: nest} 
\end{figure}
\subsection{Pursuit-Evasion game}
Pursuit-evasion games appear in many scenarios in robotics and security domains, where a team of pursuers aims to capture the evader, while the evader aims for the opposite~\cite{horak2016dynamic}. Here we consider a more realistic version of the pursuit-evasion game, \textit{NEtwork purSuiT game (NEST)}~\cite{zhang2019optimal}, which takes the tracking devices into consideration. Thus, pursuers can get the real-time location of the evader. We set a limited number of steps in which if the evader cannot escape to the exits or at least one pursuer catches the evader, the game ends and the defender gets a unit reward. Otherwise, the defender gets a zero payoff. Our experiments are conducted on grid networks in which each node connects to its neighbors. We choose some nodes on the edges as exits. The initial locations of the evader and all pursuers are chosen randomly. 
As shown in Fig.~\ref{fig: nest}(a-c), the algorithms with the mixing layer outperform other algorithms. In Fig.~\ref{fig: nest}(a), we also compare our algorithm with the algorithm that every pursuer uses Double Neural CFR independently. The result shows that Individual Double Neural CFR cannot converge to NE in a limited time. To verify the effectiveness of parameter sharing, we conduct an ablation experiment in which every pursuer uses different Agent neural networks. The results in Fig.~\ref{fig: nest}(b) indicate that CFR-MIX with parameter sharing performs better than that without parameter sharing. In Fig.~\ref{fig: nest}(c), we only compare CFR-MIX with Double Neural CFR because Deep CFR is time-consuming and needs large memory. We can see that CFR-MIX still exhibits better performance. In complicated games with large number of pursuers (Fig.~\ref{fig: nest}(d), although CFR-MIX does not converge in a limited time, it still shows the tendency of growth which means that CFR-MIX get a better strategy than random strategy (initial strategy). In conclusion, for different domains games, CFR-MIX outperforms other algorithms significantly. In large games with many players, CFR-MIX can still get some satisfactory results within a limited time while other algorithms cannot solve these large games due to the combinatorial action space.

\section{Conclusion}
In this paper, we propose a novel framework to solve the Team-Adversary Games with the combinatorial action space. Instead of using the joint action strategy representation, we represent the team's strategy with a novel individual strategy representation to reduce the strategy space. To maintain the cooperation among team players, we define a consistency relationship between these two strategy representations. Moreover, to compute the equilibrium with new strategy representation, we transform the consistency relationship between strategies to the consistency relationship between cumulative regret values. Finally, under the guarantee of these consistency, we propose the CFR-MIX framework which employs a mixing layer to implement a strategy decomposition method. Experimental results show that our algorithm significantly outperforms other state-of-the-art CFR-based algorithms. 
\section*{Acknowledgments}
This research is partially supported by Singtel Cognitive and Artificial Intelligence Lab for Enterprises (SCALE@NTU), which is a collaboration between Singapore Telecommunications Limited (Singtel) and Nanyang Technological University (NTU) that is funded by the Singapore Government through the Industry Alignment Fund – Industry Collaboration Projects Grant.

\balance
\bibliographystyle{named} 
\bibliography{references}

\clearpage
\begin{appendices}

\onecolumn
\section{Proofs of the Theorems} \label{appendix_a}
\textbf{Proof of Theorem 1}
\begin{proof}
Let $\sigma = (\sigma_\mathbb{V}, \boldsymbol{\sigma}_\mathbb{T})$ be the Nash equilibrium strategy profile and $\sigma_i(I, a_i)$ be the strategy of agent $i$. We can know that
\begin{align}
    u_{\mathbb{V}}(\sigma) \geq \max_{\sigma^{'}_\mathbb{V} \in \sum_\mathbb{V}} u_\mathbb{V}(\sigma^{'}_\mathbb{V}, \boldsymbol{\sigma}_\mathbb{T}), && u_\mathbb{T}(\sigma) \geq \max_{\boldsymbol{\sigma}^{'}_\mathbb{T} \in \sum_\mathbb{T}} u_\mathbb{V}(\sigma_\mathbb{V}, \boldsymbol{\sigma}^{'}_\mathbb{T}). \nonumber
\end{align}
From the consistency relationship between the individual action strategy and the joint action strategy, we can know that the team's equilibrium strategy satisfies the consistency relationship. Therefore, $\boldsymbol{\sigma}_\mathbb{T} = \sigma_1\sigma_2...\sigma_n$. Then, we have
\begin{align}
    u_{\mathbb{V}}(\sigma) \geq \max_{\sigma^{'}_\mathbb{V} \in \sum_\mathbb{V}} u_\mathbb{V}(\sigma^{'}_\mathbb{V}, \sigma_1\sigma_2...\sigma_n), \nonumber
    \end{align}
\begin{align}
    u_\mathbb{T}(\sigma) \geq \max_{\sigma_i^{'} \in \sum_i} u_\mathbb{V}(\sigma_\mathbb{V}, \sigma_1^{'}\sigma_2^{'}...\sigma_n^{'}). \nonumber
\end{align}
Therefore, if $\sigma = (\sigma_\mathbb{V}, \boldsymbol{\sigma}_\mathbb{T})$ is a Nash equilibrium strategy profile, then $\sigma = (\sigma_\mathbb{V},\sigma_1\sigma_2...\sigma_n)$ is the same as the Nash equilibrium strategy profile and vice versa.
\end{proof}

\noindent\textbf{Proof of Theorem 3} \cite{brown2019deep} 
\begin{proof}
We know that the establishment of Eq.(4) guarantees the consistency relationship between strategies (Theorem 2). Therefore, under Theorem 1, we know that under the new strategy representation, the Nash Equilibrium strategy keeps unchanged. Then, we can get that the regret bound for the joint strategy representation is same as the regret bound for the individual strategy representation. The another difference is that we use probe sampling algorithm and Deep CFR uses the external sampling algorithm. The following proof is similar as the proof Theorem 1 in \cite{brown2019deep} which provides the regret bound for the joint strategy representation. Here we gives a simple proof process and readers can also refer to \cite{brown2019deep} for details.  

Assume that an online learning scheme plays strategy as follows:
\begin{equation} \label{equa:strategy}
\sigma^t(I, a)=
\left\{
\begin{aligned}
\frac{y^{+}_{t}(I, a)}{\sum_a y_t^{+}(I, a)},& \quad \text{if} \quad \sum_a y_t^{+}(I, a) > 0\\
\text{arbitrary value},& \quad \text{otherwise},
\end{aligned}
\right.
\end{equation}

Corollary 3.0.6~\cite{morrill2016using} provides the upper bound of the total regret by leveraging a function of the L2 distance between $y_t^{+}$ and $R^{T,+}$ on each infoset:
\begin{align}
    \max_{a \in A}(R^T(I, a))^2 &\leq |A|\Delta^2T + 4\Delta |A| \sum_{t=1}^T\sum_{a \in A} \sqrt{(R^t_+(I, a)-y^t_+(I, a))^2} \\
    &\leq |A|\Delta^2T + 4\Delta |A| \sum_{t=1}^T\sum_{a \in A} \sqrt{(R^t(I, a)-y^t(I, a))^2} \nonumber
\end{align}
As shown in Equation \ref{equa:strategy}, $\sigma^t(I, a)$ is invariant to rescaling across all actions at an infoset, it's also the case that for any $C(I) > 0$
\begin{equation}
     \max_{a \in A}(R^T(I, a))^2 \leq |A|\Delta^2T + 4\Delta |A| \sum_{t=1}^T\sum_{a \in A}  
     \sqrt{(R^t(I, a)-C(I)y^t(I, a))^2}
\end{equation}
Let $x^t(I)$ be an indicator variable which is 1 if $I$ was traversed on iteration $t$. If $I$ was traversed then $\tilde{r}^t(I)$ was stored in $M_{V,p}$, otherwise $\tilde{r}^t(I) = 0$. Assume for now that $M_{V,p}$ is not full, so all sampled regrets are stored in the memory.

Let $\Pi^t$ be the fraction of iterations on which $x^t(I) = 1$, and let
\begin{equation}
    \epsilon^t(I) = ||E_t[\tilde{r}^t(I)|x^t(I)=1] - V(I, a|\theta^t)||_2
\end{equation}
Inserting canceling factor of $\sum_{t^{'}=1}^t x^{t^{'}}(I)$ and setting $C(I) = \sum_{t^{'}=1}^t x^{t^{'}}(I)$,
\begin{align}
    \max_{a \in A}(\tilde{R}^T(I, a))^2 &\leq |A|\Delta^2T + 4\Delta |A|\sum_{t=1}^T (\sum_{t^{'}=1}^t x^{t^{'}}(I))\sum_{a \in A} \sqrt{(\frac{\tilde{R}^t(I, a)}{\sum_{t^{'}=1}^t x^{t^{'}}(I)}-y^t(I, a))^2} \\
    &= |A|\Delta^2T + 4\Delta |A|\sum_{t=1}^T (\sum_{t^{'}=1}^t x^{t^{'}}(I)) ||E_t[\tilde{r}^t(I)|x^t(I) = 1] - V(I, a|\theta^t)||_2 \nonumber \\
    &= |A|\Delta^2T + 4\Delta |A|\sum_{t=1}^T t\Pi^t(I)\epsilon^t(I) \nonumber \\
    &\leq |A|\Delta^2T + 4\Delta |A|T\sum_{t=1}^T \Pi^t(I)\epsilon^t(I) \nonumber
\end{align}
The first term of this expression is the same as the regret bound of tabular CFR algorithm, while the second term accounts for the approximation error. In \cite{brown2019deep}, Theorem 3 shows the regret bound for $K$-external sampling, for the case of $K$-probe sampling, we can get the same results. Thus, we can get
\begin{equation}
    \max_{a \in A}(\tilde{R}^T(I, a))^2 \leq |A|\Delta^2TK^2 + 4\Delta |A|TK^2\sum_{t=1}^T \Pi^t(I)\epsilon^t(I)
\end{equation}
in this case. Following the same derivation as ~\cite{lanctot2013monte} Theorem 3, the above regret bound can lead to the bound of average regret
\begin{align}
    \overline{R}^T_p \leq \sum_{I \in \mathcal{I}_p} ((1+\frac{\sqrt{2}}{\sqrt{\rho K}})\Delta\frac{\sqrt{|A|}}{\sqrt{T}} + \frac{4}{\sqrt{T}}\sqrt{|A|\Delta \sum_{t=1}^T\Pi^t(I)\epsilon^t(I)})
\end{align}
Simplifying the first term and rearranging,
\begin{align}
    \overline{R}^T_p &\leq (1+\frac{\sqrt{2}}{\sqrt{\rho K}})|\mathcal{I}_p|\Delta \frac{\sqrt{|A|}}{\sqrt{T}} + \frac{4\sqrt{|A|\Delta}}{\sqrt{T}}\sum_{I \in \mathcal{I}_p} \sqrt{\sum_{t=1}^T\Pi^t(I)\epsilon^t(I)}) \\
    &= (1+\frac{\sqrt{2}}{\sqrt{\rho K}})|\mathcal{I}_p|\Delta \frac{\sqrt{|A|}}{\sqrt{T}} + \frac{4\sqrt{|A|\Delta}}{\sqrt{T}}|\mathcal{I}_p| \frac{\sum_{I \in \mathcal{I}_p}\sqrt{\sum_{t=1}^T\Pi^t(I)\epsilon^t(I)}}{|\mathcal{I}_p|} \nonumber \\
    &\leq (1+\frac{\sqrt{2}}{\sqrt{\rho K}})|\mathcal{I}_p|\Delta \frac{\sqrt{|A|}}{\sqrt{T}} + \frac{4\sqrt{|A|\Delta|\mathcal{I}_p|}}{\sqrt{T}} \sqrt{\sum_{t=1}^T\sum_{I \in \mathcal{I}_p}\Pi^t(I)\epsilon^t(I)} \nonumber
\end{align}
Now, let's consider the average MSE loss $\boldsymbol{L}_\mathcal{R}^T(\boldsymbol{M}_r^T)$ at time $T$ over the samples in memory $\boldsymbol{M}_r^T$. We start by stating two well-known lemmas:

1. The MSE can be decomposed into bias and variance components
\begin{equation}
    E_x[(x - \theta)^2] = (\theta - E[x])^2 + Var(\theta)
\end{equation}

2. The mean of a random variable minimizes the MSE loss  $\arg\min_{\theta}E_x[(x-\theta)^2] = E[x]$ and the value of the loss when $\theta = E[x]$ is $Var(x)$.

\begin{align}
\boldsymbol{L}_\mathcal{R}^T &= \frac{1}{\sum_{I \in \mathcal{I}_p}\sum_{t=1}^T x^t(I)} \sum_{I \in \mathcal{I}_p}\sum_{t=1}^T x^t(I)||\tilde{r}^t(I)-\mathcal{R}(I|\theta^T)||_2^2 \\
&\geq \frac{1}{|\mathcal{I}_p|T}  \sum_{I \in \mathcal{I}_p}\sum_{t=1}^T x^t(I)||\tilde{r}^t(I)-\mathcal{R}(I|\theta^T)||_2^2 \nonumber \\
&= \frac{1}{|\mathcal{I}_p|T} \sum_{I \in  \mathcal{I}_p} \Pi^T(I)E_t[||\tilde{r}^t(I)-\mathcal{R}(I|\theta^T)||_2^2|x^t(I)=1] \nonumber
\end{align}
Let $\mathcal{R}^*$ be the model that minimizes $\boldsymbol{L}^T$ on $\boldsymbol{M}_r^T$. Using above two lemmas,
\begin{align}
    \boldsymbol{L}_\mathcal{R}^T \geq \frac{1}{|\mathcal{I}_p|T} \sum_{I \in \mathcal{I}_p} \Pi^T(I) (||\mathcal{R}(I|\theta^T) - E_t[\tilde{r}^t(I)|x^t(I)=1]||_2^2 + \boldsymbol{L}_{\mathcal{R}^*}^T).
\end{align}
Thus, 
\begin{align}
     \boldsymbol{L}_\mathcal{R}^T - \boldsymbol{L}_{\mathcal{R}^*}^T \geq \frac{1}{|\mathcal{I}_p|} \sum_{I \in \mathcal{I}_p} \Pi^T(I) \epsilon^T(I) \\
     \sum_{I \in \mathcal{I}_p} \Pi^T(I) \epsilon^T(I) \leq \mathcal{I}_p|(\boldsymbol{L}_\mathcal{R}^T - \boldsymbol{L}_{\mathcal{R}^*}^T)
\end{align}
\begin{align}
    \overline{R}^T_p 
    &\leq (1+\frac{\sqrt{2}}{\sqrt{\rho K}})|\mathcal{I}_p|\Delta \frac{\sqrt{|A|}}{\sqrt{T}} + \frac{4\sqrt{|A|\Delta|\mathcal{I}_p|}}{\sqrt{T}} \sqrt{\sum_{t=1}^T\sum_{I \in \mathcal{I}_p}\Pi^t(I)\epsilon^t(I)} \\
    &\leq (1+\frac{\sqrt{2}}{\sqrt{\rho K}})|\mathcal{I}_p|\Delta \frac{\sqrt{|A|}}{\sqrt{T}} + \frac{4\sqrt{|A|\Delta|\mathcal{I}_p|}}{\sqrt{T}}\sqrt{\sum_{t=1}^T|\mathcal{I}_p|(\boldsymbol{L}_\mathcal{R}^T - \boldsymbol{L}_{\mathcal{R}^*}^T)} \nonumber \\
    &\leq (1+\frac{\sqrt{2}}{\sqrt{\rho K}})|\mathcal{I}_p|\Delta \frac{\sqrt{|A|}}{\sqrt{T}} + \frac{4\sqrt{|A|\Delta|\mathcal{I}_p|}}{\sqrt{T}}\sqrt{T|\mathcal{I}_p|\epsilon_{\boldsymbol{L}}} \nonumber \\
    &= (1+\frac{\sqrt{2}}{\sqrt{\rho K}})|\mathcal{I}_p|\Delta \frac{\sqrt{|A|}}{\sqrt{T}} + 4|\mathcal{I}_p|\sqrt{|A|\Delta\epsilon_{\boldsymbol{L}}} \nonumber
\end{align}
So far we have assumed that $\boldsymbol{M}_{r}$ contains all sampled regrets. The number of samples in the memory at iteration $t$ is bounded by $K \cdot |\mathcal{I}_p| \cdot t$. Therefore, if $K \cdot |\mathcal{I}_p| \cdot T < |\boldsymbol{M}_r|$ then the memory will never be full, and we can make this assumption. 

Let $\rho = T^{-\frac{1}{4}}$.
\begin{align}
    P( \overline{R}^T_p > (1+\frac{\sqrt{2}}{\sqrt{T^{-\frac{1}{4}} K}})|\mathcal{I}_p|\Delta \frac{\sqrt{|A|}}{\sqrt{T}} + 4|\mathcal{I}_p|\sqrt{|A|\Delta\epsilon_{\boldsymbol{L}}}) < T^{-\frac{1}{4}}
\end{align}
Therefore, for any $\epsilon > 0$,
\begin{align}
    \lim_{T \rightarrow \infty} P( \overline{R}^T_p - 4|\mathcal{I}_p|\sqrt{|A|\Delta\epsilon_{\boldsymbol{L}}} > \epsilon) = 0 
\end{align}
\end{proof}

\clearpage

\clearpage

\section{CFR-MIX Framework}

\begin{algorithm*}[ht]
\caption{CFR-MIX framework}
\begin{algorithmic}[1]
\label{algo: deepcfr}
\STATE Initialize cumulative regret network $R(I, a|\theta_p)$ with $\theta_p$ so that it returns 0 for all inputs for player $p \in \{\mathbb{V}, \mathbb{T}\}$;
\STATE Initialize average strategy network $S(I, a|\theta_{\pi, p})$ for player $p \in \{\mathbb{V}, \mathbb{T}\}$;
\STATE Initialize regret memories $M_{r,\mathbb{V}}$,$M_{r,\mathbb{T}}$ and strategy memory $M_{\pi,\mathbb{V}}$, $M_{\pi,\mathbb{T}}$.
\FOR{CFR Iteration t=1 to T}
\FOR{traverse k=1 to K}
\STATE TRVERSE($\phi, \mathbb{V}$, $\theta_\mathbb{V}$, $\theta_\mathbb{T}$,$M_{r,\mathbb{V}}$, $M_{\pi, \mathbb{T}}, t$)
\STATE TRVERSE($\phi, \mathbb{T}$, $\theta_\mathbb{T}$, $\theta_\mathbb{V}$,$M_{r,\mathbb{T}}$, $M_{\pi, \mathbb{V}}, t$) \\
\#use sample algorithm to traverse game tree and record regret and strategy into memory
\ENDFOR
\STATE Train $\theta_p$ on loss for player $p \in \{\mathbb{V}, \mathbb{T}\}$ \\ $\mathcal{L} = \mathbb{E}_{(I, \widetilde{r}) \sim M_{r, p}}[\sum_a ((R(\cdot|\theta_p^t)+\widetilde{r})^+-R(\cdot|\theta_p^{t+1}))^2]$
\STATE Train $\theta_{\pi, p}$ on loss for player $p \in \{\mathbb{V}, \mathbb{T}\}$ \\ $\mathcal{L} = \mathbb{E}_{(I, \widetilde{\pi}) \sim M_{\pi, p}}[\sum_a ((S(\cdot|\theta_p^t) + \widetilde{\pi})^+- S(\cdot|\theta_p^{t+1}))^2]$
\ENDFOR
\RETURN $\theta_{\pi,\mathbb{V}}$, $\theta_{\pi, \mathbb{T}}$
\end{algorithmic}
\end{algorithm*}

\begin{algorithm*}[ht]
\caption{TRVERSE}
\begin{algorithmic}[1]
\label{traverse}
\STATE \textbf{Function:}TRVERSE($h, p,$ $\theta_p$, $\theta_{-p}$,$M_{r,p}$, $M_{\pi, (-p)}$)
\IF {$h \in Z$}
\RETURN $u_i(h)$
\ELSIF{$h$ is a chance node}
\STATE Sample an action $a$ from the probability $\sigma_c(h)$;
\RETURN TRVERSE($ha, p,$ $\theta_p$, $\theta_{-p}$,$M_{r,p}$, $M_{\pi, (-p)}$)
\ELSIF{$P(h) = p$}
\STATE $I\leftarrow$ Information set containing $h$;
\STATE $\sigma^t(I) \leftarrow$ Strategy of Information set $I$ computed from $R(I, a|\theta_p)$ using regret matching+;
\STATE Sample an action $a^*$ with the probability $1/|A(I)|$ of each action 
\FOR{$a \in A(I)$}
\IF {$a = a^*$}
\STATE $u(a) \leftarrow$ TRVERSE($ha^*, p,$ $\theta_p$, $\theta_{-p}$,$M_{r,p}$, $M_{\pi, (-p)}$)
\ELSE
\STATE $u(a) \leftarrow$ PROBE($ha, p,$ $\theta_p$, $\theta_{-p}$,$M_{r,p}$, $M_{\pi, (-p)}$) 
\ENDIF
\ENDFOR
\STATE $u_{\sigma^t} \leftarrow \sum_{a \in A(I)} \sigma^t(I,a)u(a)$
\FOR{$a \in A(I)$}
\STATE $r(I, a) \leftarrow u(a) - u_{\sigma^t}$
\ENDFOR
\STATE Insert the infoset and its action regret values $(I, t, r(I))$ into regret memory $M_{r,p}$
\RETURN $u_{\sigma^t}$
\ELSE
\STATE $I\leftarrow$ Information set containing $h$;
\STATE $\sigma^t(I) \leftarrow$ Strategy of Information set $I$ computed from $R(I,a|\theta_{-p})$ using regret matching+;
\STATE Insert the infoset and its strategy $(I, t, \sigma^t(I))$ into strategy memory $M_{\pi, (-p)}$;
\STATE Sample an action $a$ from the probability distribution $\sigma^t(I)$;
\RETURN TRVERSE($ha, p,$ $\theta_p$, $\theta_{-p}$,$M_{r,p}$, $M_{\pi, (-p)}$)
\ENDIF
\end{algorithmic}
\end{algorithm*}

\begin{algorithm*}[t]
\caption{PROBE}
\begin{algorithmic}[1]
\label{probe}
\STATE \textbf{Function:}PROBE($h, p,$ $\theta_p$, $\theta_{-p}$,$M_{r,p}$, $M_{\pi, (-p)}$)
\IF {$h \in Z$}
\RETURN $u_i(h)$
\ELSIF{$h$ is a chance node}
\STATE Sample an action $a$ from the probability $\sigma_c(h)$;
\ELSE
\STATE $I\leftarrow$ Information set containing $h$;
\STATE $i \leftarrow$ the player who takes an action at the history $h$; 
\STATE $\sigma^t(I) \leftarrow$ Strategy of Information set $I$ computed from $R(I, a|\theta_i)$ using regret matching+;
\STATE Sample an action $a$ from the probability distribution $\sigma^t(I)$;
\ENDIF
\RETURN PROBE($ha, p,$ $\theta_p$, $\theta_{-p}$,$M_{r,p}$, $M_{\pi, (-p)}$)
\end{algorithmic}
\end{algorithm*}

\end{appendices}
\end{document}